\documentclass[12pt]{article}

\usepackage{amsmath,amsfonts,amssymb,amsthm,graphicx,color,subfigure}
\usepackage[normalem]{ulem}
\usepackage{xcolor}
\usepackage[utf8]{inputenc}
\usepackage{verbatim}
\usepackage{algorithmic}
\usepackage{algorithm}
\usepackage{multirow}
\usepackage{tikz}
\usepackage{tikz-qtree}
\usepackage{url}
\usepackage{cite}
\usetikzlibrary{shadows,trees}

\DeclareGraphicsExtensions{.pdf,.png,.jpg}

\theoremstyle{definitions}
\newtheorem{theorem}{Theorem}
\newtheorem{lemma}{Lemma}

\newtheorem{definition}{Definition}

\newtheorem{corollary}{Corollary}

\renewcommand{\algorithmiccomment}[1]{\bgroup\hfill//~#1\egroup}

\def\ccl{\textsc{C3L}}

\def\sgn{\mathrm{sign}}
\def\ln{\mathrm{ln}}

\def\R{\mathbb{R}}

\def\G{\mathcal{G}}
\def\Y{\mathcal{Y}}

\def\k{p^\alpha}

\newcommand\mm[1]{\hat{m}_{#1}}
\renewcommand\ss[1]{\hat{\sigma}_{#1}}
\renewcommand\SS[1]{\hat{\Sigma}_{#1}}

\newcommand\m[2]{m_{#1}^{#2}}
\newcommand\s[2]{\sigma_{#1}^{#2}}
\renewcommand\S[2]{\Sigma_{#1}^{#2}}

\newcommand\ssub[1]{_{#1}}
\newcommand\sssub[2]{_{#1:#2}}
\newcommand\usub[1]{{[#1]}}
\newcommand\ussub[2]{{[#1:#2]}}

\def\cl{\mathrm{class}}

\def\tr{\mathrm{tr}}
\def\det{\mathrm{det}}
\def\nor{\mathcal{N}}

\title{Semi-supervised model-based clustering with controlled clusters leakage}
\author{Marek \'Smieja \and \L{}ukasz Struski \and Jacek Tabor}
\date{\normalsize $^{1}$Faculty of Mathematics and Computer Science\\ Jagiellonian University\\ \L{}ojasiewicza 6, 30-348 Krakow, Poland}

\begin{document}

\maketitle

\begin{abstract}
In this paper, we focus on finding clusters in partially categorized data sets. We propose a semi-supervised version of Gaussian mixture model, called {\ccl}, which retrieves natural subgroups of given categories. In contrast to other semi-supervised models, {\ccl} is parametrized by user-defined leakage level, which controls maximal inconsistency between initial categorization and resulting clustering. Our method can be implemented as a module in practical expert systems to detect clusters, which combine expert knowledge with true distribution of data. Moreover, it can be used for improving the results of less flexible clustering techniques, such as projection pursuit clustering. The paper presents extensive theoretical analysis of the model and fast algorithm for its efficient optimization. Experimental results show that {\ccl} finds high quality clustering model, which can be applied in discovering meaningful groups in partially classified data.
\end{abstract}

\section{Introduction}

Model-based clustering aims at finding a mixture of probability models, which optimally estimates true probability distribution on data space. Contrary to other clustering techniques, it does not only recover meaningful groups, but also gives a rule (probability model) for generating elements from clusters. Therefore, it is commonly used in various areas of machine learning and data analysis \cite{wehrens2004model, salah2016model, spurek2017general}.% such as image segmentation, natural language processing, object recognition, etc. 

Although clustering is an unsupervised technique, one can introduce additional information to guide the algorithm what is the expected structure of clusters. Semi-supervised learning methods usually use partial labeling \cite{liu2015clustering} or pairwise constraints \cite{lu2007semi} to transfer expert knowledge into clustering process, while consensus and alternative clustering gather information from several partitions of data into one general view \cite{nguyen2007consensus, gondek2007non}.
%One classical type of side information is partial labeling, which gives a division of a small portion of data into categories. Another way relies on defining pairwise constraints: must-link constraints indicate pairs of data instances that originate from the same class, while cannot-link constraints are used to mark examples that come from different categories. Both partial labeling and pairwise constraints require human intervention. In contrast, consensus clustering gathers information from several partitions of data into one general view, while alternative clustering aims at finding groups which provide a perspective on the data that expands on what can be inferred from previous clusterings.
In this paper, we assume that we have the knowledge about division of data set into two categories and focus on the following problem: {\em How to find the best model of clusters that preserves a fixed amount of information about existing categories?} In other words, we focus on finding interesting clusters, which are very likely to belong to one category.

To explain a basic motivation behind our model, let us consider an expert system used for automatic text translation. It is a common practice to construct several translation models, each designed for one cluster retrieved from a data set \cite{aggarwal2012mining}. Alternatively, since texts are often categorized into specific domains, e.g. sport, politics, etc., then each translator can be fitted to one of these categories. To consider together both options, we could implement a separate module responsible for finding clusters, which (a) are described by compact models (e.g. Gaussians) and (b) are related with predefined topics. Observe that optimization of these two conflicting goals simultaneously is non-trivial. We cannot cluster elements from each category individually, because this strategy does not lead to optimal solution for the entire data set (in terms of likelihood). Moreover, existing categorization might be inaccurate as well as the interesting groups can cross the boundary between predefined domains. Therefore, a better approach is to incorporate the constraint to the clustering process and always work with the entire data set.

Our method can also be applied to strictly unsupervised situations, where no initial categorization is given. Let us recall that one way to analyze clusters in complex data spaces relies on finding projections onto one dimensional subspaces, where groups can be easily identified. Projection pursuit focuses on choosing such a direction, which optimizes selected statistical index such as kurtosis \cite{pena2001cluster} or skewness \cite{loperfido2013skewness}. Since one dimensional views generate linear decision boundaries in original data space, it is not possible to find flexible cluster structures. However, we can input such a linear boundary to our model in order to improve existing clusters. Our method directly uses the information from initial splitting, but can extend linear decision surfaces to nonlinear ones generated by probabilistic mixture models.

\begin{figure*}[t]
\centering
\subfigure[$\alpha = 0.159$]{\label{a:jeden}\includegraphics[width=1.6in]{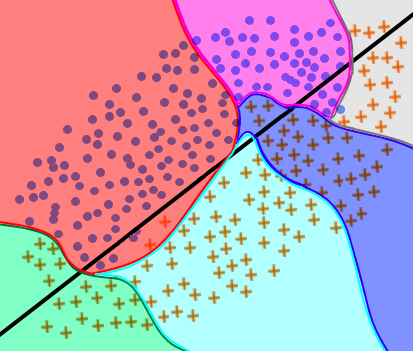}}\quad
\subfigure[$\alpha = 0.043$]{\label{a:dwa}\includegraphics[width=1.6in]{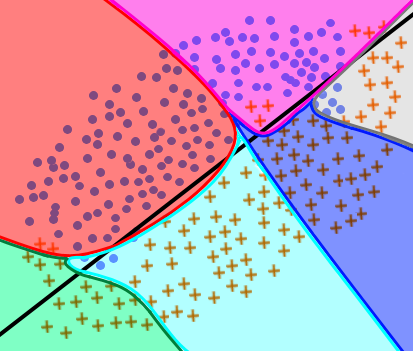}}\quad
\subfigure[$\alpha = 0.001$]{\label{a:dwa}\includegraphics[width=1.6in]{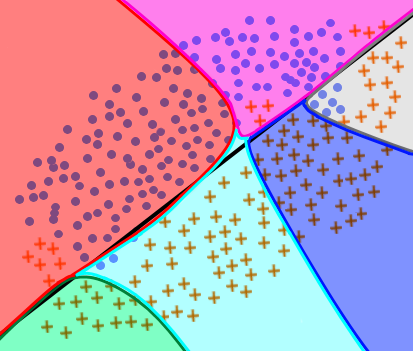}}
\caption{The effects of {\ccl} for different values of the leakage level $\alpha$.} \vspace{-0.2in}
\label{conv2}
\end{figure*}

Following the above motivation, we propose a semi-supervised clustering with controlled clusters leakage model ({\ccl}), which integrates a distribution of data with a fixed division of the space into two categories. {\ccl} focuses on finding a type of Gaussian mixture model (GMM) \cite{mclachlan2004finite}, which maximizes the likelihood function and preserves the information contained in the initial splitting with a predefined probability (leakage level). Intuitively, we allow for the flow of clusters densities over decision surface, but with a full control of total probability assigned to the opposite category, which is defined as the leakage level $\alpha \in (0,1)$ (see Figure \ref{conv2}). This general idea is formulated as a constrained optimization problem (Section \ref{sec:idea}). %similar to the one in projection pursuit clustering \cite[Section 2]{pena2001cluster}.

The advantages of {\ccl} can be summarized as follows:
\begin{enumerate}
\item It has a closed form solution in a special case of cross-entropy clustering (a type of GMM) \cite{tabor2014cross}.
\item It can be efficiently implemented and optimized by a modified on-line Hartigan algorithm (Section \ref{sec:opt}). 
\item The user can directly parametrize {\ccl} by a maximal inconsistency level between initial categorization and final clustering model (leakage level).
\item The selection of the leakage level $\alpha$ allows to move from a strictly unsupervised GMM for $\alpha = 0.5$, where decision boundary has no effect on clustering, to the limiting case of $\alpha \to 0$, where every group is fully condensed in one category (Section \ref{thoertical}).
\end{enumerate}

Experimental studies confirm that the proposed approach builds a high quality model under a given constraint in terms of inner clustering measures, such as Bayesian Information Criterion (Section \ref{sec:ex1}). It can be successfully used to discover meaningful groups in partially classified data (Section \ref{sec:ex2}) as well as to improve existing clusters obtained by applying projection techniques (Section \ref{sec:ex3}). We present a real-life case study, in which the use of {\ccl} allows to detect subgroups of chemical space given their division into active and inactive classes (Section \ref{exp:last}).

\section{Related work}

Semi-supervised clustering incorporates the knowledge about class labels to partitioning process \cite{basu2008constrained}. This information can be presented as partial labeling, which gives a division of a small portion of data into categories, or as pairwise constraints, which indicate whether two data points originate from the same (must-links) or distinct classes (cannot-links). Although pairwise constraints provide less amount of information than partial labeling, it is easier to assess whether two instances come from the same group than assign them to particular classes.

Clustering with pairwise constraints was introduced by Wagstaff et al. \cite{wagstaff2001constrained}, who created a variant of k-means, which focuses on preserving all constraints. Shental et al. \cite{shental2004computing} constructed a version of Gaussian mixture model, which gathers data points into equivalence classes (called chunklets) using must-link relation and then applied EM algorithm on such generalized data set of chunklets. This approach was later modified to multi-modal clustering models \cite{smieja2016constrained}. The aforementioned methods work well with noiseless side information, but deteriorate the results when some constraints are mislabeled. To overcome this problem, the authors of~\cite{basu2004probabilistic, Lu_PPC} applied hidden Markov random fields (HMRF) to construct more sophisticated dependencies between linked points. However, the use of HMRF leads to complex solutions, which are difficult to optimize. In recent years, Asafi and Cohen-Or. \cite{Asafi_ConstraintsFeatures} suggested reducing distances between data points with a must-link constraint and adding a dimension for each cannot-link constraint. After updating all other distances to, e.g., satisfy the triangle inequality, the thus obtained pairwise distance matrix can be used for unsupervised learning. Wang and Davidson\cite{Wang_FCSC} proposed a version of spectral clustering, which relies on solving a generalized eigenvalue problem.

Partial labeling is used in clustering to define sample data points from particular classes. Liu and Fu \cite{liu2015clustering} added additional attributes to feature vectors and proposed modified k-means algorithm. There is also a semi-supervised version of fuzzy c-means~\cite{pedrycz1997fuzzy, pedrycz2008fuzzy}, where the authors supplied the cost function with a regularization term that penalizes fuzzy partitions that are inconsistent with the side information. GMMs can be adapted to make use of class labels by combining the classical unsupervised GMM with a supervised one \cite{ambroise2001learning, zhu2009introduction}. 

Since assigning data points to classes or labeling pairwise constraints requires extensive domain knowledge, then many clustering methods were adapted to use additional information about data, which does not require human intervention. One example is consensus clustering, which considers gathering information coming from different domains \cite{nguyen2007consensus}. On the other hand, complementary (alternative) clustering aims at finding groups which provide a perspective on the data that expands on what can be inferred from previous partitions~\cite{gondek2007non}. 

{\ccl} is a version of Gaussian mixture model, which uses side information given by class labels or more generally by a decision boundary between classes. In contrast to classical methods applying partial labeling, it focuses on finding subgroups of original classes. This goal is similar to information bottleneck method \cite{tishby2000information, chechik2005information}. Roughly speaking, this approach tries to construct compact clusters (compressed representation), which contain high amount of information about existing classes (auxiliary variable). While information-theoretic approaches use mutual information (or conditional entropy) to preserve the consistency with an initial categorization, {\ccl} explicitly defines maximal probability of inconsistency (leakage level). The leakage level can be understood as Bayes error in classification or significance level in hypothesis testing and restricts every cluster model to be assigned to one of two initial classes with a predefined probability. Subgroups could also be detected by using cannot-link constraints, clustering with pairwise constraints does not allow to input maximal level of error. Moreover, computational complexity of applying cannot-link constraints to GMM is usually high, while {\ccl} works in a comparable time to classical unsupervised mixture models.

{\ccl} can be naturally combined with projection pursuit approach, which focuses on selecting low dimensional projections of data for finding clusters (or other meaningful characteristics of data). Such a projection can be determined by optimizing selected statistical coefficients e.g. kurtosis \cite{pena2001cluster, hou2014re} or skewness \cite{loperfido2013skewness, loperfido2015vector}. Since every one dimensional view induces linear decision boundary in the original space, this technique may not be sufficient to detect complex data patters. {\ccl} allows to take such a rough linear splitting of data and correct simple decision boundaries to nonlinear ones.

\section{Theoretical model} \label{sec:idea}

In this section, we introduce our model and discuss its possible extensions and applications.

Let a data space $\R^N$ be divided by a codimension one hyperplane\footnote{$(N-1)$-dimensional hyperplane} $H$ given by:
$$
H = \{x \in \R^N: h^Tx = a\},
$$
for fixed $h \in \R^N$ and $a \in \R$. The hyperplane $H$ induces hard classification rule: the class label of each point $x \in \R^N$ is determined by
\begin{equation}\label{class}
\cl_H(x)=\sgn(h^Tx-a).
\end{equation}
This splits a dataset $X \subset \R^N$ into two groups $X_+,X_-$ given by 
\begin{equation} \label{eq:division}
X_\pm=\{x \in X: \cl_H(x)=\pm 1\}.
\end{equation}
Alternatively, we can consider the initial classification of entire space $\R^N$ as
$$
H_\pm=\{x \in \R^N:\cl_H(x)=\pm 1\},
$$
and put $X_\pm=X \cap H_\pm$. Note that the uncertainty of class label is usually higher for elements localized closer to the barrier than for those with larger distance from $H$.

In a model-based clustering we focus on estimating a density of data space with a use of mixture of $k$ densities, $g=\sum_{i=1}^k p_i g_i$, where every $g_i$ belongs to a given parametric family of densities (usually Gaussian) and $p_i$ are prior probabilities \cite{mclachlan2004finite}. This goal can be practically realized by maximizing the likelihood function. Let us define the inconsistency between a cluster density $g_i$ and the initial classification:
$$
\alpha_i = \min \{ \int_{H_-} g_i(x) dx, \int_{H_+} g_i(x) dx \}.
$$
The above formula gives the amount of probability that is spread to opposite class and it is related with Bayes error of Gaussian model assuming that $H$ predicts the class membership correctly.

Our question is: how to find a clustering model that optimizes a likelihood function and provides high consistency with initial classification? If we knew that $H$ gives a perfect classification rule, then we could try to keep every model $g_i$ maximally consistent with $H$, i.e. perform a separate clustering of every category. However, this is usually not the case and this strategy does not guarantee to obtain optimal solution (in terms of likelihood) for the entire data set. Moreover, some interesting groups can cross the decision boundary. Therefore, we should allow for the flow of corresponding densities over a decision surface, but with the full control of the total probability assigned to the opposite class. In our approach, we formulate a constrained optimization problem, where we aim at finding such a mixture model $g$ that maximizes the quality of density estimation and preserves a fixed inconsistency level $\alpha$, i.e. every component $g_i$ has to satisfy $\alpha_i \leq \alpha$. 

One could probably try to realize the above goal by a classical GMM approach, however, at very high numerical and theoretical cost. It would be impossible to get a closed form solution and complex non-linear optimization would be needed. Therefore, in this paper we have decided to use cross-entropy clustering (CEC) \cite{tabor2014cross, spurek2017r, smieja2015spherical}, which similarly to GMM divides data with respect to Gaussian distributions. Contrary to GMM, in CEC the clusters do not ``cooperate'' one with another to build the global cost function\footnote{Instead of optimizing a density $p_i g_1 + \ldots+ p_k g_k$, CEC finds optimal subdensity $\max\{p_i g_1, \ldots, p_k g_k\}$, i.e. every point $x$ is linked with exactly one model $g_i$.} and consequently it is enough to calculate the cost function for each cluster individually. 

CEC is based on Minimum Description Length Principle (MDLP) \cite{rissanen1985minimum} and focuses on minimizing the generalized cross-entropy function, by selecting optimal Gaussian probability distribution for each cluster\footnote{The minimization of cross-entropy is equivalent to the maximization of likelihood function.}. Given a single cluster $X$ and corresponding density $g$, the empirical cross-entropy function equals:
$$
h^\times(X\|g) = - \frac{1}{|X|} \sum_{x \in X} \ln (g(x)).
$$
In the case of Gaussian densities, $g=\nor(m,\Sigma)$, the above
formula can be reduced to its closed form:
\begin{equation}\label{crossNormal}
\begin{array}{l}
h^{\times}(X\|\nor(m,\Sigma))= \\
\tfrac{N}{2}\ln(2\pi)+
\tfrac{1}{2}\|\mm{X}-m\|_{\Sigma}+\tfrac{1}{2}\tr(\Sigma^{-1}\SS{X})+\tfrac{1}{2}\ln \det (\Sigma),
\end{array}
\end{equation}
where 
\begin{equation} \label{meanCovN}
\begin{array}{l}
\mm{X} := \frac{1}{|X|} \sum\limits_{x \in X} x, \\
\SS{X}  := \frac{1}{|X|} \sum\limits_{x \in X} (x  - \mm{X}) (x  - \mm{X})^T,
\end{array}
\end{equation}
denote the sample mean and covariance of $X$ and 
$$
\|x\|_{\Sigma} := x^T \Sigma^{-1} x
$$
is the Mahalanobis norm. Given $k$ clusters $X_1,\ldots,X_k$ the overall cross-entropy function equals
\begin{equation}\label{eq:CECcost}
\sum_i p_i (-\ln p_i) + p_i h^\times(X_i\|g_i),
\end{equation}
where $g_i$ is a Gaussian density with parameters given by \eqref{meanCovN} for $X_i$ and $p_i = \frac{|X_i|}{|X|}$.

The term $p_i (-\ln p_i)$ adds a cost for maintaining a cluster. In consequence, the method tends to keep the model simple and allows for the reduction of redundant clusters. Therefore, CEC cost function \eqref{eq:CECcost} combines the model accuracy with its complexity.

With this, we are ready to define our {\ccl} model. First, we give a definition of a linear constraint, which restricts every cluster model to one category with a fixed probability.
\begin{definition}
Let $H$ be a codimension one hyperplane on $X \subset \R^N$ and let $\alpha > 0$. We say that a density $g$ satisfies a linear constraint $(H,\alpha)$, if 
\begin{equation} \label{eq:model}
\text{either }\int_{H_-}g(x)dx\geq 1-\alpha \text{ or }
\int_{H_+}g(x)dx\geq 1-\alpha.
\end{equation}
The number $\alpha$ will be referred as the leakage level.
\end{definition}
The above definition of linear constraint is analogical to linear separability with power $\alpha$ given in \cite[Section 2]{pena2001cluster} in the context of projection pursuit.

If $\alpha \geq \frac{1}{2}$ then an arbitrary density satisfies one of the conditions given by \eqref{eq:model}. Therefore, {\em a strict constraint} is given by $\alpha  < \frac{1}{2}$. Observe that the above constraint is reminiscent of a typical approach used in hypothesis testing, where we accept a given hypothesis if it lies within a predefined percentage of density.
{\em In our case we consider only those density cluster models, which lie with a probability $(1-\alpha)$ on one side of decision boundary.} 

We now introduce a linear constraint to CEC framework. First, we define the criterion function for a single cluster:
\begin{definition} (One cluster {\ccl} cost function) \label{oneCost}
Let $X \subset \R^N$ be divided by a codimension one hyperplane $H$. Given $\alpha >0$ and a family $\G$ of Gaussian densities, {\ccl} cost function for $X$ is defined by 
\begin{equation}\label{eq:onecluster}
\begin{array}{l}
E^\alpha_H(X\|\G):=\\
\inf \{h^{\times}(X\|g) : \, g \in \G \text{ which satisfies constraint } (H,\alpha)\},
\end{array}
\end{equation}
where $h^{\times}(X\|g)$ is given by \eqref{crossNormal}.
\end{definition}
We always assume that a covariance matrix of $X$ is nonsingular, i.e., $\det (\SS{X}) \neq 0$. This prevents from the situation when $X$ lies in the subspace of $\R^N$, which might lead to degenerate solutions.

The overall {\ccl} cost is defined as follows:
\begin{definition} (Overall {\ccl} cost function for clustering) \label{moreCost} Let $X \subset \R^N$ be divided by a codimension one hyperplane $H$. Given a splitting of $X$ into clusters $X_1,\ldots,X_k$, a family $\G$ of Gaussian densities and $\alpha > 0$, the total {\ccl} clustering cost equals 
\begin{equation} \label{eq:total}
\begin{array}{l}
E^\alpha_H(X_1,\ldots,X_k\|\G)=\\
\sum_{i=1}^k p_i (-\ln p_i+E^\alpha_H(X_i\|\G)) \text{, where }p_i=\frac{|X_i|}{|X|}.
\end{array}
\end{equation}
\end{definition}
The optimal clustering is the one, which minimizes the above cost function (the optimization problem will be the subject of the next section). We emphasize that {\ccl} accepts arbitrary Gaussian densities as mixture components, which satisfy the linear constraint. In particular, each Gaussian can have distinct covariance matrix.

Let us observe that introduced model can be applied in the case of any decision boundary (not only linear hyperplane). If $f_+$ and $f_-$ are two decision functions that quantify the chance of assigning data points to positive and negative classes, then the label of an instance $x \in \R^N$ is chosen as 
\begin{equation} \label{discr}
\cl(x) = \mathrm{arg} \max\limits_{j=\pm} f_j(x).
\end{equation}
Clearly, for two class problem we can define one discriminant $f = f_+ - f_-$. In consequence, the formula \eqref{discr} can be simplified to:
\begin{equation}\label{discr2}
\cl(x) = \sgn f(x).
\end{equation}

Next, we extend the input space $\R^N$ to $\R \times \R^N$ and embed our data set $X$ into this space by:
$$
X \ni x \to (f(x),x) \in \R \times \R^N.
$$ 
Observe that a hyperplane $H = \{0\} \times \R^N$ gives the same classification rule in $\R^{N+1}$ to the formula \eqref{discr2} in $\R^N$. 

Let $\G_N$ denotes the set of all Gaussian densities on $\R^N$ and let $\G_{1,N}$ be the set of Gaussian densities, that can be factorized into two independent components, defined by:
\begin{equation}\label{eq:den}
\G_{1,N}:=\{ g(x) = g_1(x\ssub{1}) \cdot g_{N}(x\sssub{2}{N+1}) \text{ : } g_1 \in \G_1,g_{N} \in \G_{N}\},
\end{equation}
where $x\sssub{k}{l}=(x_k,\ldots,x_l)$, for $x=(x_1,\ldots,x_{N+1})$. If we consider a mixture model $g = \sum_i p_i g_i$, where $g_i = g^i_1 \cdot g^i_N \in \G_{1,N}$, then the first component $g^i_1$ will describe a distribution of discrimination function $f$ while $g^i_{N}$ will describe a density in the original space $X$. Therefore, the use of $\G_{1,N}$ allows to model a distribution of data and discriminant function individually. We will use the family $\G_{1,N}$ in the next section.

\section{Optimization} \label{sec:opt}

Without loss of generality, we assume that a decision boundary of $X \subset \R^N$ is given by a hyperplane\footnote{Observe that given an arbitrary hyperplane one can always shift the original data and change the basis of $\R^N$ in an orthonormal way to obtain this situation.} $H=\{0\} \times \R^{N-1}$. From now on, our attention is restricted to the class of Gaussian densities $\G_{1,N-1}$ defined by \eqref{eq:den}, i.e. we assume that every component $g$ is of the form $g = g_1 \cdot g_{N-1}$, where $g_1$ is $1$-dimensional and $g_{N-1}$ is $(N-1)$-dimensional Gaussian density. This model suits perfectly to the case of arbitrary decision boundary described at the end of section \ref{sec:idea}, where we model the data distribution and the values of decision support function separately.

%%%%%%%%%%%%%%%%%%%%%%%%%%%%%%%%%%%%%%%%%%%%%%%%

We are going to show that in the case of $\G_{1,N-1}$ we can compute one cluster cost function analytically. Let us first observe that the selection of 1-dimensional density $g_1 \in \G_1$ and $(N-1)$-dimensional density $g_{N-1} \in \G_{N-1}$ can be done separately in {\ccl} clustering. First, we verify that the linear constraint is independent of $g_{N-1}$, i.e.,
$$
\begin{array}{l}
\int_{[0,+\infty)\times \R^{N-1}} g(x)  dx= \int_0^{+\infty} g_1(x\ssub{1}) dx\ssub{1},\\ 
\int_{(-\infty,0]\times \R^{N-1}} g(x) dx = \int_{-\infty}^0 g_1(x\ssub{1}) dx\ssub{1}.
\end{array}
$$
Then, observe that the cross-entropy between $X$ and $g$ can be calculated as:
$$
h^\times(X \| g) =  h^\times(X\usub{1} \| g_1) + h^\times(X\ussub{2}{N} \| g_{N-1}),
$$
where $X\ussub{k}{l}$ denotes a data set $X$ restricted to the attributes from $k$ to $l$. This follows from
$$
\begin{array}{l}
h^\times(X \| g)  \\
=\sum \limits_{x \in X} -\ln g(x)= \sum\limits_{x \in X} -\ln (g_{1}(x\ssub{1}) \cdot g_{N-1}(x\sssub{2}{N})) \\[1ex]
=  \sum\limits_{x \in X} \big( -\ln (g_{1}(x\ssub{1}))- \ln(g_{N-1}(x\sssub{2}{N})) \big) \\[1ex]
= \sum\limits_{x\ssub{1} \in X\usub{1}} -\ln (g_{1}(x\ssub{1}))+
\sum\limits_{x_{N-1} \in X\ussub{2}{N}}- \ln(g_{N-1}(x\ssub{N-1})) \\[1ex]
= h^\times(X\usub{1} \| g_1) + h^\times(X\ussub{2}{N} \| g_{N-1}).
\end{array}
$$

Since the product densities can be selected individually, the optimization subject to the constraint is performed in one dimension. 
\begin{corollary} \label{cor:reduction}
Let $(H,\alpha)$ be a linear constraint defined on dataset $X \subset \R^N$, where $H=\{0\} \times \R^{N-1}$. Then the one cluster {\ccl} cost function of $X$ is given by:
$$
E^\alpha_H(X \| \G_{1,N-1}) = E^{\alpha}_{\{0\}}(X\usub{1}\|\G_1)+
h^\times(X\ussub{2}{N}\|\G_{N-1}),
$$
where $E^{\alpha}_{\{0\}}((X\usub{1}\|\G_1)$ is one cluster {\ccl} cost function \eqref{eq:onecluster} calculated in one dimensional situation. 
\end{corollary}

To complete the formula given in Corollary \ref{cor:reduction}, the {\ccl} cost function in one dimensional case has to be calculated. To facilitate the calculation we first give an equivalent form of linear constraint. 

Given $\alpha > 0$, let us denote by $\k$ the corresponding quantile:
\begin{equation}\label{palpha}
\k := \Phi_{\nor(0,1)}^{-1}(1-\alpha),
\end{equation}
where $\Phi_{\nor(m,\sigma)}(\cdot)$ denotes a cumulative distribution function of $\nor(m,\sigma)$. Making use of elementary calculations we get that:
$$
\alpha = \Phi_{\nor(m,\sigma)}(m-\k \sigma),
$$
for any $m \in \R$ and $\sigma > 0$. Then, one dimensional density $\nor(m,\sigma)$ satisfies the constraint $(\{0\},\alpha)$, iff
\begin{equation} \label{c1}
|m| \geq \k \sigma,
\end{equation}
In other words, the distance between the mean $m$ and the barrier is at least $\k \sigma$. %Equivalently, coming back to the motivation from hypothesis testing, we accept the density with normal distribution $\nor(m,\sigma^2)$ if it satisfies
%$$
%0 \not\in (m-\k \sigma,m+\k \sigma).
%$$

To calculate the optimal one dimensional cost function, we must observe that the
cross-entropy \eqref{crossNormal} between a distribution of dataset $X \subset \R$ with mean $\mm{X}$ and standard deviation $\ss{X}$ and a Gaussian density $\nor(m,\sigma)$ equals:
\begin{equation}\label{cfg-norm}
\begin{array}{l}
h^{\times}(X \| \nor(m,\sigma)) = \frac{1}{2}\left( \frac{\ss{X}^2 + (m-\mm{X})^2}{\sigma^2} + \ln(\sigma^2) + \ln(2\pi) \right).
\end{array}
\end{equation}
The optimal parameters of $\nor(m,\sigma)$ are obtained by the minimization of the above function under the restriction \eqref{c1}:

\begin{theorem}\label{thm}
Let $X \subset \R$ be a dataset with the mean $\mm{X} \neq 0$ and the standard deviation $\ss{X} > 0$. We assume that $(\{0\},\alpha)$ denotes the linear constraint on $X$, for $\alpha >0$,  and $\k = \Phi^{-1}_{\nor(0,1)}(1-\alpha)$.

If $|\mm{X}| \geq \k\ss{X}$, then put $\m{X}{\alpha} := \mm{X}$, $\s{X}{\alpha} := \ss{X}$, otherwise
\begin{equation} \label{eq:optUnivariate}
\begin{array}{ll}
\m{X}{\alpha} := & \frac{-(\k)^2 \mm{X}+\sgn(\mm{X}) \k\sqrt{((\k)^2+4)\mm{X}^2 + 4\ss{X}^2}}{2},\\ \s{X}{\alpha} :=& \frac{|\mm{X}^\alpha|}{\k}.
\end{array}
\end{equation}

Then, the normal density $\nor(\m{X}{\alpha},\s{X}{\alpha})$ minimizes the value of $h^{\times}(X \| \nor(m,\sigma))$, given by \eqref{cfg-norm}, under the restriction $|m| \geq \k\sigma_X$ ({\ccl} cost function $E^{\alpha}_{\{0\}}(X \| \G)$).
\end{theorem}

\begin{proof}
Our aim is to find the minimum of the function
\begin{equation} \label{eq:ro}
h(m,\sigma)=h^\times(X\|\nor(m,\sigma))
\text{ under the condition }|m| \geq \k\sigma.
\end{equation}
It is obvious that the above function has the derivative zero only at its global minimum which is given by a pair $(m_X, \sigma_X)$. Consequently, if $m=\mm{X},\sigma=\ss{X}$ satisfies the constraint $|m| \geq \k\sigma$, then we have found the minimum.

In the opposite case, we only need to verify what happens on the boundary of the constraints (since we do not have any local minimum inside $|m|>\k\sigma$), that is when $\k\sigma=|m|$. 
Then, by \eqref{cfg-norm}, the function \eqref{eq:ro} simplifies to
$$
h(m) =\frac{1}{2}\big( \frac{ \ss{X}^2 + (m - \mm{X})^2}{m^2} (\k)^2 + \ln \frac{m^2}{(\k)^2} +\ln(2\pi)\big).
$$
Then
$$
h'(m)= -\frac{\ss{X}^2 + \mm{X}^2}{m^3}(\k)^2 + \frac{\mm{X}}{m^2}(\k)^2 + \frac{1}{m}.
$$
Finally, the solution of $h'(m)=0$ which minimizes the value of $h(m)$, is given by
$$
m = \frac{-(\k)^2 \mm{X}+\sgn(\mm{X})\k\sqrt{((\k)^2+4)\mm{X}^2 + 4\ss{X}^2}}{2},
$$
which completes the proof.
\end{proof}

The above analysis shows how to calculate the best model of clusters for a given partition. However, finding an optimal partition is NP-hard problem, where heuristic iterative algorithms are commonly used. One can apply a slight modification of Hartigan approach to optimize {\ccl} cost function (see \ref{app:algo} for details). Similar algorithm is used in optimization of CEC and k-means methods.

\section{Theoretical analysis} \label{thoertical}

In this section we present a theoretical analysis of {\ccl} model in its simplified form. We start with investigating the convergence of cluster parameters with respect to the leakage level $\alpha$. Then, we show that under certain assumptions a decision boundary determined by {\ccl} model with two clusters converges to the initial barrier, when $\alpha$ approaches to $0$.

In order to accommodate the constraint one-dimensional density cluster model modifies its mean and standard deviation according to Theorem \ref{thm}. The relation between $\m{X}{\alpha}$ and $\s{X}{\alpha}$ (given by \eqref{eq:optUnivariate})  is inversely proportional, i.e., the increase of $\m{X}{\alpha}$ results in the decrease of $\s{X}{\alpha}$ and vice versa (see Figure \ref{fig:asymptotic}). 
However, the most important fact is that $|\m{X}{\alpha}|$ does not grow infinitely, but converges to a finite number dependent on a data set. To prove it formally, let us first consider a one dimensional case.
\begin{figure*}[t]
\centering
\subfigure[]{\label{fig:alpha_p}\includegraphics[width=1.5in]{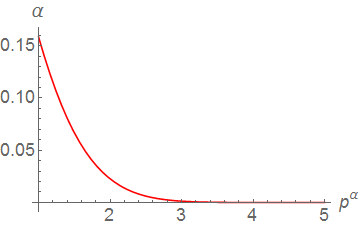}}\quad
\subfigure[]{\label{a:jeden}\includegraphics[width=1.5in]{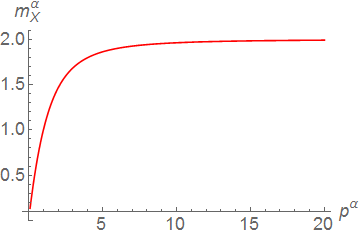}}\quad
\subfigure[]{\label{a:dwa}\includegraphics[width=1.5in]{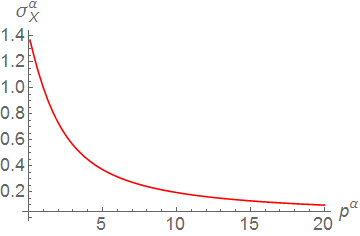}} 
\caption{The influence of the leakage level $\alpha$ on the parameters $\k, \m{X}{\alpha}$ and $\s{X}{\alpha}$.} \vspace{-0.2in}
	\label{fig:asymptotic} 
\end{figure*}

\begin{lemma} \label{malyLem}
We assume that $X \subset \R$ is a data set with a mean $\mm{X} \neq 0$ and standard deviation $\ss{X} > 0$. Let $g^\alpha = \nor(\m{X}{\alpha}, \s{X}{\alpha})$ denote a density minimizing one cluster {\ccl} cost function under the linear constraint $(\{0\}, \alpha)$, i.e, $\m{X}{\alpha}$ and $\s{X}{\alpha}$ are given by Theorem \ref{thm}.

Then:
\[ \left.
\begin{array}{ll}
\m{X}{\alpha} & \to \mm{X} + \frac{\ss{X}^2}{\mm{X}} \\
\s{X}{\alpha} & \to 0 \\
\end{array}
\right.  \text{, as } \alpha \to 0. \] 
\end{lemma}

\begin{proof}
The proof is included in \ref{app:lem}.
\end{proof}

If we combine the above result with the fact that the mean and the covariance of $(N-1)$ dimensional density $g_{N-1}$, for a model $g=g_1 \cdot g_{N-1}$,  do not depend on the linear constraint, but are the maximum likelihood estimators of data, we get the following corollary:

\begin{corollary} \label{maleTw}
We assume that $X \subset \R^N$ is a data set, with a mean $\mm{X}$ and a covariance matrix $\SS{X}$, where $\mm{X\usub{1}} \neq 0$ and $\ss{X\usub{1}}^2=\SS{X\usub{1}}$. Let $g^\alpha = g_1^\alpha \cdot g_{N-1} \in \G_{1,N-1}$ denote a density minimizing one cluster {\ccl} cost function under the linear constraint $(\{0\}\times \R^{N-1}, \alpha)$, i.e., $g^\alpha_1 = \nor(\m{X\usub{1}}{\alpha}, \s{X\usub{1}}{\alpha})$ is given by Theorem \ref{thm} and $g_{N-1} = \nor(\m{X\ussub{2}{N}}{\alpha}, \S{X\ussub{2}{N}}{\alpha})$.

Then:
\[ 
\begin{array}{ll}
\m{X\usub{1}}{\alpha} & \to \mm{X\usub{1}} + \frac{\ss{X\usub{1}}^2}{\mm{X\usub{1}}}, \\
\s{X\usub{1}}{\alpha} & \to 0, \\
\m{X\ussub{2}{N}}{\alpha} & =\mm{X\ussub{2}{N}}, \\
\S{X\ussub{2}{N}}{\alpha} & =\SS{X\ussub{2}{N}}, \\
\end{array}
\text{ as } \alpha \to 0.
 \]
\end{corollary}
The Figure \ref{fig:border} shows the influence of the change of the parameter $\alpha$ on the form of resulting density function.

\begin{figure}[t]
\centering
\includegraphics[width=3.5in]{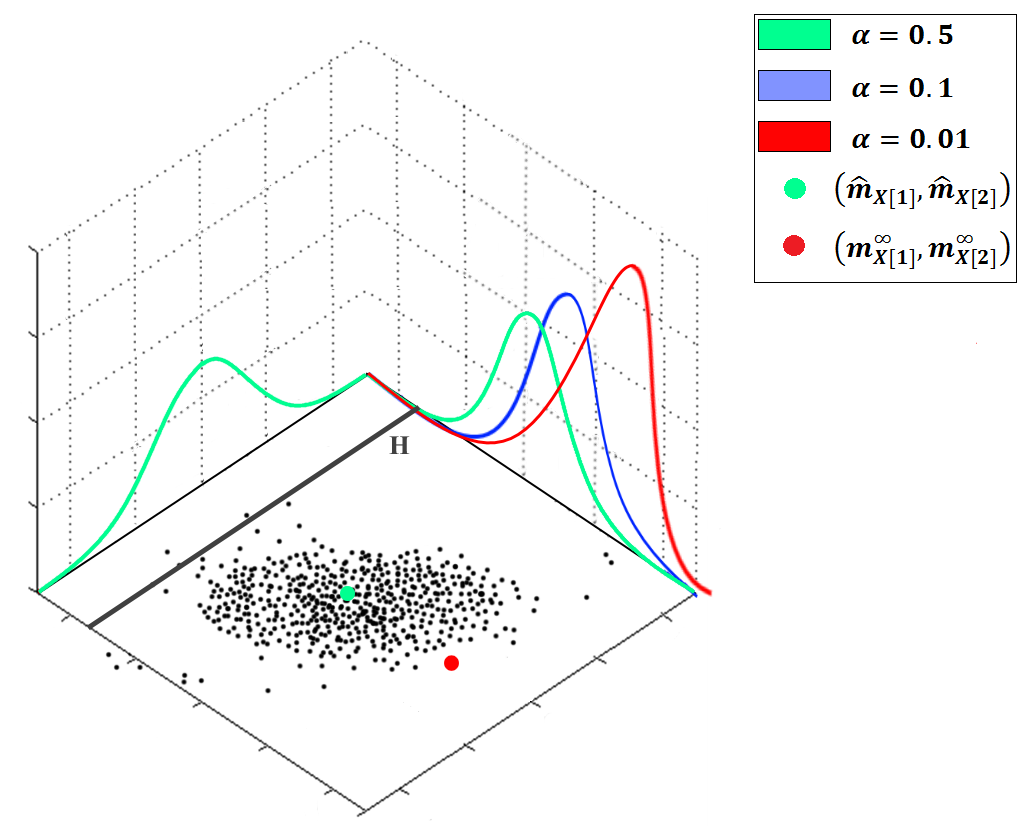}
\caption{The change of $\alpha$ affects only  the form of a density model which is orthogonal to the barrier. Its mean converges to the limiting value (marked with red dot) given by Corollary \ref{maleTw} while its standard deviation grows to infinity.}\vspace{-0.2in}
	\label{fig:border} 
\end{figure}

We now discuss the relations between an initial decision boundary defined by a hyperplane $H$ and a splitting determined by {\ccl} model. For a simplicity, we consider the case of only two clusters. We assume that the hyperplane $H=\{0\} \times \R^{N-1}$ defines the classification rule \eqref{class}:
$$
\cl_H(x) := \sgn(x\ssub{1}) \text{, for } x \in \R^N,
$$
which splits a data set $X \subset \R^N$ into two subsets $X_\pm$ \eqref{eq:division}.

Let the leakage level $\alpha$ be fixed. In a simplified form of {\ccl} model with only two clusters, we assume that density models $p_\pm g^\alpha_\pm$ are chosen so that to maximize one cluster cost functions of $X_\pm$, respectively (not the overall cost of clustering). In the other words, a class density is selected ignoring the influence of objects which belong to the opposite class. Then any incoming object $x \in X$ can be classified to one of two clusters by calculating
$$
\cl_\alpha(x) := \sgn (p_+ g_+^\alpha(x) - p_-g_-^\alpha(x)) = \pm 1,
$$
We will show that a decision boundary determined by such model converges to $H$, as the leakage level $\alpha$ approaches to 0, i.e., 
$$
\cl_\alpha(x) \xrightarrow{\alpha \to 0} \cl_H(x),
$$ 
which is illustrated in Figure \ref{conv}:

\begin{figure*}[t]
\centering
\subfigure[$\alpha = 0.159$]{\label{a:jeden}\includegraphics[width=1.6in]{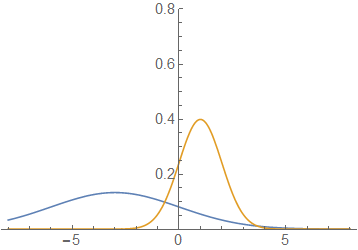}}\quad
\subfigure[$\alpha = 0.043$]{\label{a:dwa}\includegraphics[width=1.6in]{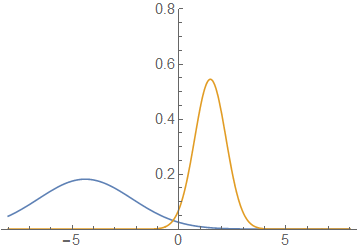}}\quad
\subfigure[$\alpha = 0.001$]{\label{a:dwa}\includegraphics[width=1.6in]{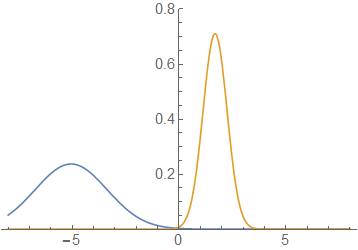}}
\caption{Convergence of {\ccl} model for $\alpha \to 0$.}
\label{conv}\vspace{-0.2in}
\end{figure*}

%\begin{figure}[t]
%\centering
%\fbox{\includegraphics[width=1.5in]{discrimination}}
%\caption{Gaussian classification. Although each discriminant was learned on examples with different signs, the blue density covers most of the area of positive numbers.} \label{artefactt}
%\end{figure}

\begin{theorem} \label{mainTh}
We assume that $X \subset \R^N$ is a data set and $H = \{0\} \times \R^{N-1}$ defines a hyperplane in $\R^N$ dividing $X$ into two classes $X_-,X_+$, where $X_\pm \neq \emptyset$. Let $g^\alpha_\pm \in \G_{1,N-1}$ denote two densities minimizing one cluster {\ccl} functions of $X_\pm$ under the constraint $(H,\alpha)$, respectively.

Then, there exists a constant $C > 0$ such that
$$
\cl_\alpha(x) \to \cl_H(x) \text{, as } \alpha \to 0,
$$
for every $x \in \R^N$ satisfying $dist(x, H) \leq C$.
\end{theorem}

The above convergence holds only for instances, which are localized within the margin of size $C$ around the barrier $H$. This is a natural situation occurring in every Gaussian discrimination. To prove the above theorem we first consider one dimensional situation, where an exact value of constant $C$ will be given.

\begin{lemma} \label{lemLast}
We assume that $X_\pm \subset \R_\pm$ are two non empty sets with means $\mm{\pm}$ and standard deviations $\ss{\pm}$. Let $g^\alpha_\pm \in \G_1$ denote two densities minimizing one cluster {\ccl} cost functions of $X_\pm$ under the linear constraint $(\{0\},\alpha)$, respectively.

Then, 
$$
\ln g^\alpha_+(x) -\ln g^\alpha_-(x) \to +\infty \text{, as } \alpha \to \infty,
$$
for every $0<x \leq 2 \left( \mm{+} + \frac{\ss{+}^2}{\mm{+}} \right)$. 
\end{lemma}

\begin{proof}
The proof is included in \ref{app:lem1}.
\end{proof}

In one dimensional case the constant $C$ from Theorem \ref{mainTh} equals $C= 2 \left( \mm{+} + \frac{\ss{+}^2}{\mm{+}} \right)$. Below we complete the proof of our main result:

\begin{proof} (of Theorem \ref{mainTh})
Let $x \in \R^N$ be such that $0<x\ssub{1}<2 \left( \mm{+} + \frac{\ss{+}^2}{\mm{+}} \right)$, where $\mm{+}, \ss{+}$ are the mean and standard deviation of $X_+\usub{1}$. In other words, we assume that $x$ lies at the right side of a decision boundary. We assume that the optimal {\ccl} densities $g^\alpha_\pm \in \G_{1,N-1}$ for $X_\pm$ equal
$$
g^\alpha_\pm = (g^\alpha_\pm)_1 \cdot (g^\alpha_\pm)_{N-1} 
$$
with the priors $p_\pm$. We will show that 
\begin{equation}\label{ineq}
\ln p_+g^\alpha_+(x) > \ln p_-g^\alpha_-(x),
\end{equation} 
for sufficiently small $\alpha > 0$. 

Since
$$
\begin{array}{ll}
\ln g^\alpha_\pm(x) & = \ln (g^\alpha_\pm)_1(x\ssub{1}) + \ln (g_\pm^\alpha)_{N-1}(x\sssub{2}{N}) + \ln p_\pm,
\end{array}
$$
the formula \eqref{ineq} can be rewritten as 
\begin{equation}\label{koniec}
\begin{array}{l}
\ln (g^\alpha_+)_1(x\ssub{1})  - \ln (g^\alpha_-)_1(x\ssub{1}) +\ln (g^\alpha_+)_{N-1}(x\sssub{2}{N}) \\
- \ln (g^\alpha_-)_{N-1}(x\sssub{2}{N}) + \ln p_+ - \ln p_-> 0.
\end{array}
\end{equation}

Making use of Lemma \ref{lemLast}, we have 
$$
\ln (g^\alpha_+)_1(x\ssub{1})  - \ln (g^\alpha_-)_1(x\ssub{1}) \to +\infty \text{, as } \alpha \to 0.
$$
Because 
$$
|\ln (g^\alpha_+)_{N-1}(x\sssub{2}{N}) - \ln (g^\alpha_-)_{N-1}(x\sssub{2}{N})  + \ln p_+ - \ln p_-| < \infty,
$$ 
the LHS of \eqref{koniec} can be arbitrary large, when $\alpha \to 0$, which completes the proof.
\end{proof}

\section{Experiments} \label{sec:exp}

We tested our method on sample examples retrieved from UCI repository \cite{asuncion2007uci} and one real data set of chemical compounds \cite{warszycki2013linear}. We verified the quality of the clustering model and demonstrated that {\ccl} can be useful in discovering natural subgroups given a partial knowledge about two class division. We also used {\ccl} to extend linear boundary between clusters obtained by projection pursuit technique to nonlinear one. We also show its application on real data set of chemical compounds. We compared its performance with related model-based clustering techniques. 

\subsection{Quality of the model} \label{sec:ex1}

In this experiment we consider a scenario, where every instance is assigned to one of two classes based on the value of a fixed decision support function. {\ccl} builds a clustering model which preserves the information of class membership in a sense that every cluster density belongs to one of two classes with a probability greater than $(1-\alpha)$ \eqref{eq:model}. We want to verify the quality of such model and compare it with the results produced by related model-based clustering techniques, which however do not allow for a direct specification of the leakage level.

To compare the quality of clustering models we applied Bayesian Information Criterion (BIC), which is a standard criterion for model selection \cite{fraley1998many}.  %It is defined by:
%$$
%BIC = -2 \log \hat{L} + k \log(n),
%$$
%where $L$ is likelihood function of the model, $k$ is the number of free parameters and $n$ is the number of data points. 
The lower the BIC is the better the model is.

\begin{table}[t]
\caption{Regression UCI data sets and median number of clusters returned by {\ccl}.}
\label{tab:reg}	
\setlength{\arrayrulewidth}{0.1mm}
\setlength{\tabcolsep}{4pt}
\renewcommand{\arraystretch}{1.1}
\centering
\begin{tabular}{lcccc}
 & Airfoil & Forest & Music$^+$ & Stock\\ \hline
\# Instances & 1502 &  517 & 1059 & 536 \\ 
\# Features & 5 & 12 & 5 & 7\\ 
\# Clusters & 5 & 3 & 5 & 7\\ 
\hline
\end{tabular}	

\footnotesize $^+$ PCA was used to reduce a dimension of data
\end{table}

We used four regression UCI examples, which are summarized in Table \ref{tab:reg}. A dependent (output) variable was treated as a decision support function, which determines a decision boundary $H$. More precisely, if $X \times Y$ is a data set, where $X \subset \R^N$ contains explanatory variables and $Y \subset \R$ includes dependent variable, then a decision boundary $H$ is defined by $H=X \times \{\mathrm{median}(Y)\}$, where $\mathrm{median}(Y)$ is the median of attribute $Y$.

The effects of {\ccl} were compared with those obtained by classical GMM method, which ignores the presence of existing decision boundary. To introduce a decision boundary to the model, we also considered the second variant of GMM (which is referred to as GMM$_H$): given a linear hyperplane $H$, which divides a data set $X$ into two regions $X_-, X_+$ (see \eqref{eq:division}), GMM$_H$ is defined as follows:
\begin{itemize}
\item GMM is applied to $X_-$ and $X_+$ separately, which give two models $g_\pm=p_1^\pm g_1^\pm+\ldots+p_k^\pm g_k^\pm$.
\item These models are combined into a single one by $g=\frac{|X_-|}{|X|} g_- + \frac{|X_+|}{|X|} g_+$. 
\item Finally, every point $x \in X$ is assigned to the most probable cluster by calculating $\frac{|X_\pm|}{|X|} p_i^\pm g_i^\pm(x)$.
\end{itemize}
Analogical strategies were also applied to CEC. Both variants of GMM use general Gaussian densities to model clusters distributions, while CEC-based methods use the same densities as {\ccl} method, i.e. densities from the family $\G_{1,N-1}$.

\begin{figure*}[t]
\centering
\subfigure[Airfoil]{\label{fig:h}\includegraphics[width=2.3in]{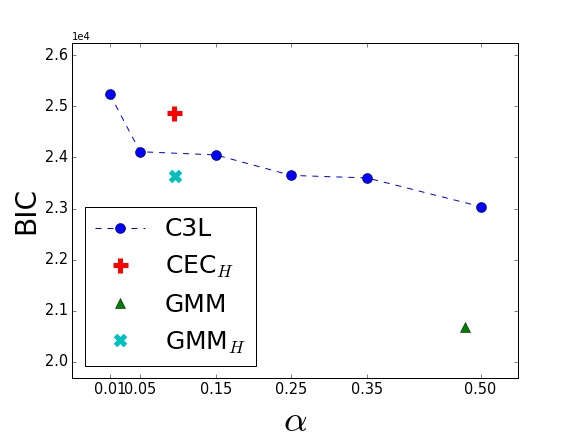}}
\subfigure[Forest]{\label{fig:ari}\includegraphics[width=2.3in]{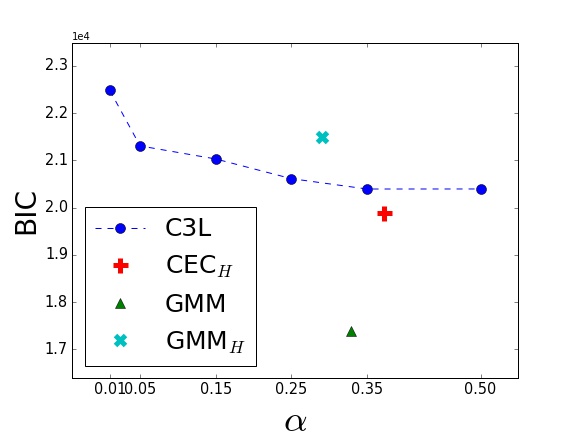}}
\subfigure[Music]{\label{fig:ll}\includegraphics[width=2.3in]{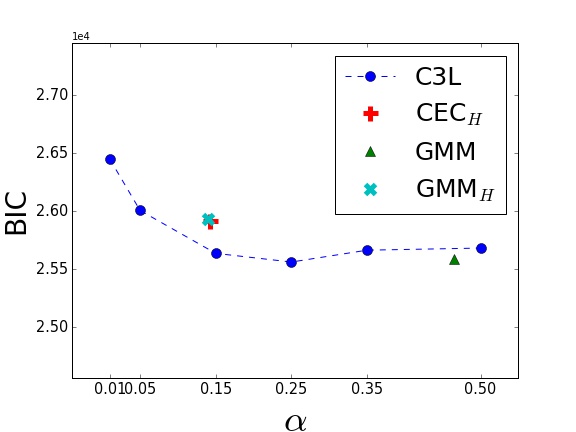}}
\subfigure[Stock]{\label{fig:ari}\includegraphics[width=2.3in]{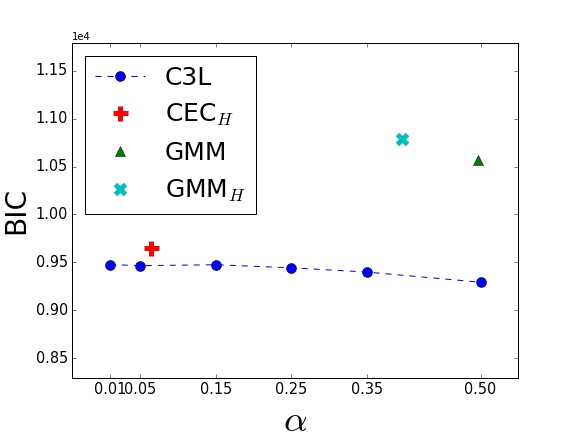}}
\caption{Quality of the clustering models measured by Bayesian Information Criterion (BIC).}\vspace{-0.2in}
	\label{fig:reg} 
\end{figure*}

We ran each method on $X \times Y$ with a decision boundary $H$. To investigate the influence of the leakage level on the clustering effects of {\ccl}, six leakage levels were considered, $\alpha \in \{0.01, 0.05, 0.15, 0.25, 0.35, 0.5\}$. Since {\ccl} and CEC$_H$ internally find the final number of clusters, we ran them with 10 groups, while other methods used the median number of groups returned by {\ccl} calculated over these six leakage levels.

The results presented in the Figure \ref{fig:reg} prove that the quality of {\ccl} model improves as the leakage level is increased. This is a natural behavior, because lower values of $\alpha$ indicate higher restrictions on the clusters models. Since other methods do not control the inconsistency with classification, we measured their resulting leakage levels and marked returned BIC values. One can observe that in most cases CEC$_H$ and GMM$_H$ gave worse BIC than {\ccl} method for corresponding leakage levels. It follows from the fact that {\ccl} optimizes the model on the entire data, while CEC$_H$ and GMM$_H$ search for the optimal solutions in each half space individually. Most importantly, since we are not able to directly control the inconsistency level of these methods, they might lead to high inconsistency with initial classification, even if each model is optimized on a separate class (see Forest and Stock data sets). Similar argument holds for classical GMM -- although it should allow for optimal fit to the data, it does not take into account the decision boundary between classes.

\subsection{Subgroups detection}\label{sec:ex2}

\begin{table}
\caption{UCI data sets used in subgroups detection experiment. Last two rows show which reference groups were used for creating classes $Y_-$ and $Y_+$.}
\vspace{0.1in}
\label{tab:clust}	
\setlength{\arrayrulewidth}{0.1mm}
\setlength{\tabcolsep}{3pt}
\renewcommand{\arraystretch}{1.1}
\centering
\begin{tabular}{ccccc}
 & Balance & Segmentation$^+$ & User & Wine \\ \hline
\# Instances & 625 & 210 & 258 & 178 \\ 
\# Features & 4 & 5 & 5 & 13\\ 
\# Clusters  & 3 & 7 & 4 & 3 \\ 
$Y_-$ & \{1,2\} & \{1,3,4,5,7\} & \{1,2\} & \{1,2\} \\ 
$Y_+$ & \{3\} & \{2,6\} & \{3,4\} & \{3\}\\
\hline
\end{tabular}	

\footnotesize $^+$ PCA was used to reduce a dimension of data
\end{table}

Decision boundary usually delivers some meaningful information about true structure of clusters. For example, the User Knowledge Modeling data set \cite{kahraman2013development} distinguishes users with very low, low, middle and high knowledge about a given subject. If we knew a coarse separation of the users into two basic classes, \{very low, low\} and \{middle, high\}, it should be easier to detect their exact level of knowledge. 
%For example, instances of UCI Glass data set are partitioned into three-level hierarchy tree. Given a partial knowledge about the binary division into window and non-window glasses at the first hierarchy level, we may want to discover their six subgroups from the lowest tree level. 
In this experiment, we want to verify how the information of binary classification influences the clustering results.

To simulate the above scenario, where a decision boundary is closely related with the expected clustering structure, we considered four UCI data sets. For each one we applied the following procedure:
\begin{itemize}
\item Given $k$ reference groups $Y_1,\ldots,Y_k$ of a data set $X \subset \R^N$ we created two classes $Y_-,Y_+$ by merging selected groups together. 
\item We trained SVM classifier on 15\% elements drawn randomly from $Y_-$ and $Y_+$, which induced a linear decision boundary $H$ dividing $X$ into two classes $X_-$ and $X_+$.
\end{itemize}
The goal is to discover the reference grouping $Y_1, \ldots, Y_k$. Table \ref{tab:clust} contains detailed information about data sets and classes $Y_-,Y_+$. 

\begin{table}
\caption{Normalized mutual information for UCI datasets.}
\vspace{0.1in}
\label{tab:nmi}	
\setlength{\arrayrulewidth}{0.1mm}
\setlength{\tabcolsep}{4pt}
\renewcommand{\arraystretch}{1.1}
\centering
\begin{tabular}{lcccc}
Method & Balance  & Segmentation & User &  Wine \\ \hline
{\ccl}$_{0.01}$ & \bf 0.50 & \bf 0.62 & 0.53 & 0.50 \\
{\ccl}$_{0.05}$ & 0.44 & 0.60 & 0.49 & \bf 0.51 \\
CEC$_H$ & 0.48 & 0.59 & 0.47 & 0.34 \\
CEC & 0.03 & 0.58 & 0.25 & 0.46 \\
c-GMM & 0.20 & 0.54 & 0.56 & 0.47 \\
GMM$_H$ & 0.49 & 0.56 & \bf 0.66 & 0.40 \\
GMM & 0.07 & 0.58 & 0.36 & 0.45\\
\hline \vspace{-0.3in}
\end{tabular}	
%seg 17, glass 41
\end{table}

%In practice a decision boundary (or initial labeling) might be inaccurate. Therefore, to investigate the influence of the accuracy of the initial decision boundary on the clustering results, we also considered the second variant, where the incorrect class label was assigned with probability $1/5$ to every of 15\% of labeled data. In this case, SVM should usually give worse separation of classes.

Given such prepared data sets, we ran all the methods applied in previous experiment. Additionally, we used a version of GMM enhanced with pairwise cannot-link constraints, which is referred as c-GMM \cite{shental2004computing}. Cannot-link constraints specify the pairs of elements that should not be included into the same group, which suits perfectly to this clustering task\footnote{The introduction of must-link constraints is not suitable in this case.}. To generate a set of pairwise constraints containing similar knowledge to the decision boundary $H$, we went over all pairs of labeled data points and generated a cannot-link constraint, if one element belonged to $Y_-$ and the second belonged to $Y_+$.

To compare {\ccl} with other methods we used only two leakage levels $\alpha = 0.01$ and $\alpha = 0.05$. This choice was motivated by a typical approach used in hypothesis testing, where the significance level is commonly set to $0.01$ or $0.05$. {\ccl}, CEC and CEC$_H$ were initialized with twice the correct number of clusters (and they were allowed to reduce redundant groups). GMM-based methods were run with the correct numbers of clusters and can thus be expected to perform better than {\ccl}, especially that GMMs describe the clusters by arbitrary Gaussian distributions. The similarity between the obtained clusterings and the ground truth partition $Y_1,\ldots,Y_k$ was evaluated using Normalized Mutual Information (NMI) \cite{ana2003robust}. NMI is bounded from the above by the value 1, which is attained for identical partitions.

The results presented in Table \ref{tab:nmi} show that {\ccl} performed better than other methods except the User Knowledge Modeling data set, where GMM$_H$ gave very good result. This confirms that working with the entire data set is usually more profitable than finding subgroups in each half space individually (as GMM$_H$ and CEC$_H$ do). Moreover, lower leakage level $\alpha =0.01$ usually led to higher NMI than $\alpha = 0.05$. It might follow from the fact that a decision boundary was constructed based on correctly labeled data and, therefore, it was very accurate.

To investigate the influence of the accuracy of decision boundary on the clustering results, we used 15\% of data drawn from $Y_-,Y_+$ and assigned incorrect labels to a fixed percentage of them (we considered 0\%, 10\%, 20\% and 30\% of erroneous labels). The more labels were misspecified the worse a decision boundary should be. 

Table \ref{tab:cor} presents the correlation between the accuracy of decision boundary and the normalized mutual information of clustering. One can observe that CEC$_H$ and GMM$_H$ are more sensitive to incorrect decision boundary than both parameterizations of {\ccl}. In consequence, these methods should not be used if there is a risk of unreliable information of class labels. Higher robustness of {\ccl} could be explained by the fact that this model has an access to the entire data set, not only to its part (as GMM$_H$ and CEC$_H$).

\begin{table}
\caption{Correlation between the accuracy of decision boundary and clustering results.}
\vspace{0.1in}
\label{tab:cor}	
\setlength{\arrayrulewidth}{0.1mm}
\setlength{\tabcolsep}{4pt}
\renewcommand{\arraystretch}{1.1}
\centering
\begin{tabular}{lcccc}
Method & Balance & Segmentation & User & Wine \\ \hline
{\ccl}$_{0.01}$ & 0.95  & 0.64 & 0.98 & 0.45 \\
{\ccl}$_{0.05}$ & 0.90  & 0.34 & 0.94 & 0.36 \\
CEC$_H$ & 0.99 & 0.68 & 0.96 & 0.76 \\
GMM$_H$ & 0.99 & 0.98 & 0.88 & 0.80 \\
%c-GMM & 0.93 & 0.27 & 0.81 & ?\\
\hline \vspace{-0.3in}
\end{tabular}	
%seg 17, glass 41
\end{table}

\subsection{Improving clusters boundaries} \label{sec:ex3}

Projection pursuit (PP) is a technique used for analyzing multivariate data by finding its interesting low dimensional views. One dimensional projection is a common choice, which was widely analyzed in the literature \cite{hou2014re, loperfido2015vector}. Projections are usually determined by maximizing non-gaussianity. Pena and Prieto \cite{pena2001cluster} showed that minimizing the kurtosis coefficient implies maximizing the bimodality of the projections, which in consequence is useful for detecting clusters. Since clustering in one dimensional subspaces can only generate linear decision boundaries between clusters, PP cannot discover complex data patterns. We will show that given linear separation of data obtained by PP, the use of {\ccl} allows to detect more accurate shapes of clusters.

\begin{figure}[t]
\centering
\subfigure[ ]{\label{fig:PP}\includegraphics[width=1.6in]{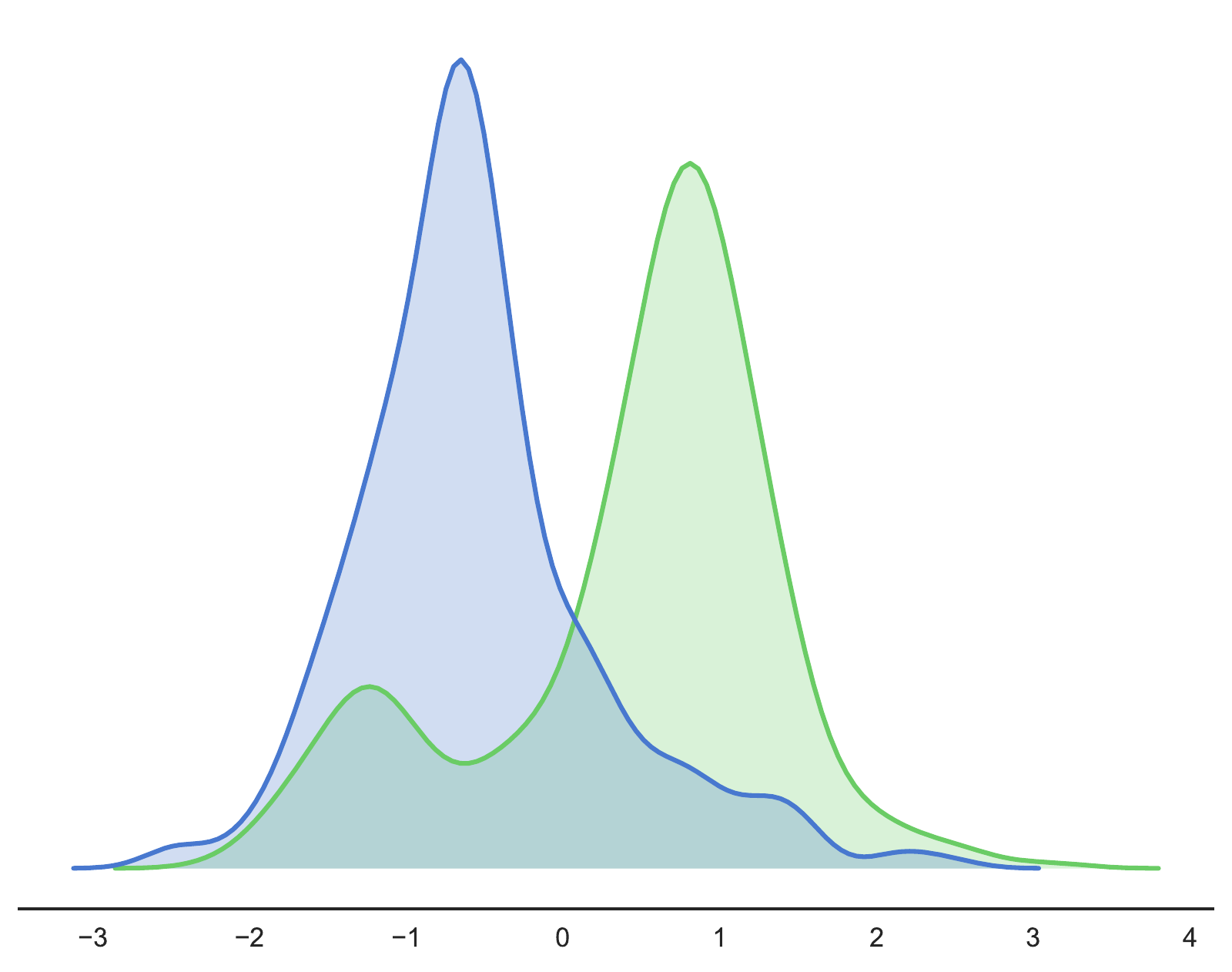}}
\subfigure[]{\label{fig:PPGMM}\includegraphics[width=1.6in]{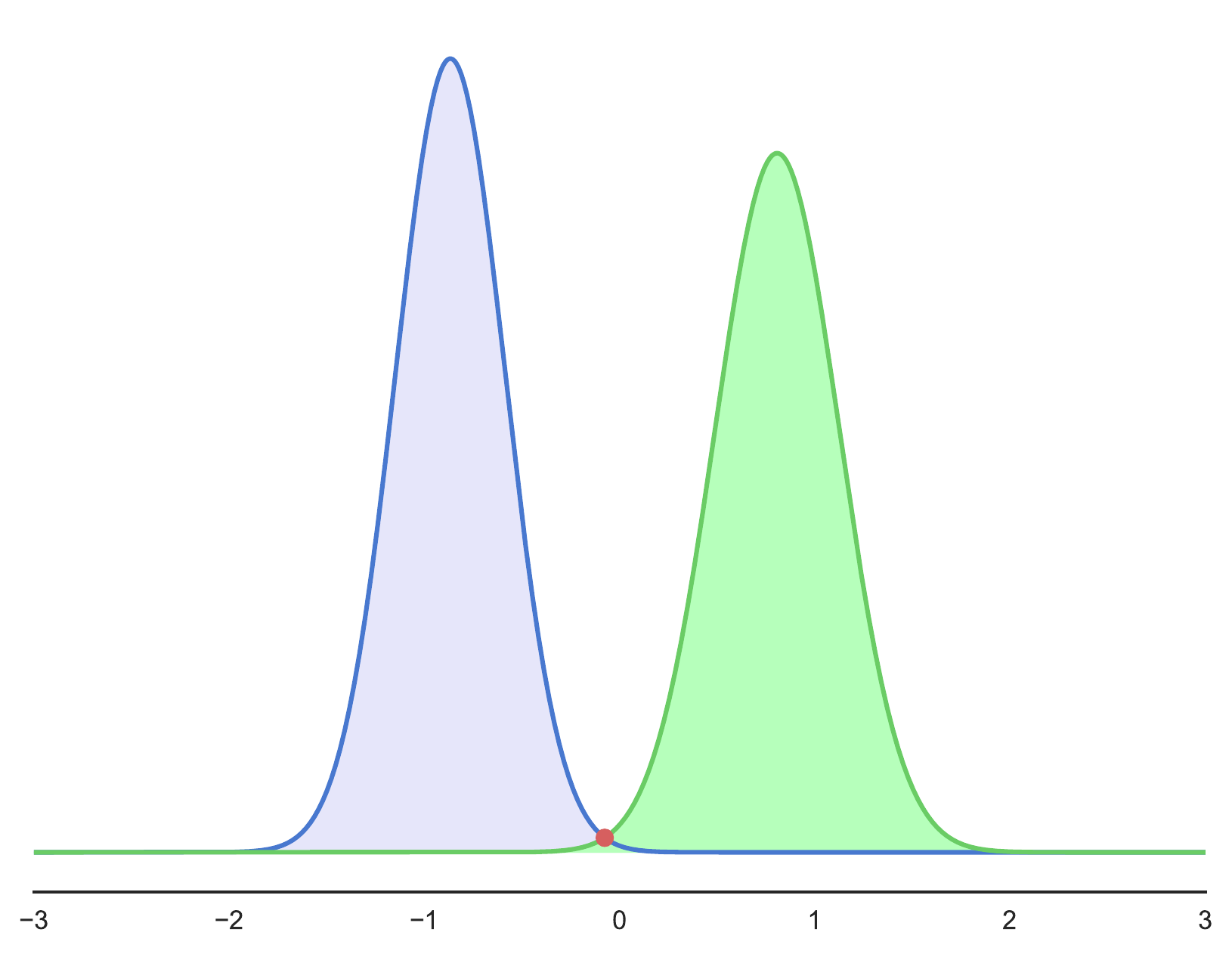}}
\subfigure[]{\label{fig:PPc3l}\includegraphics[width=1.6in]{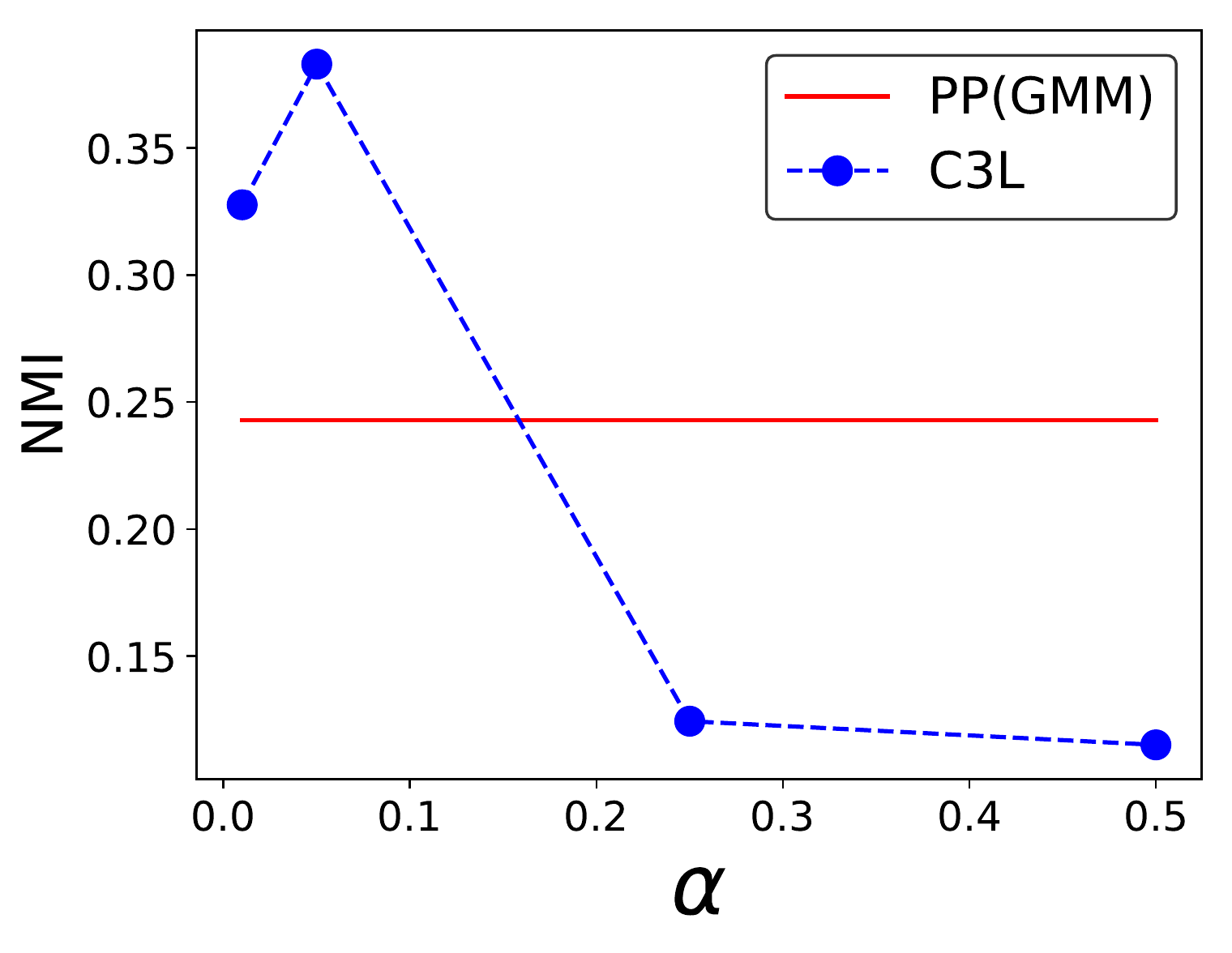}}
\caption{{\ccl} with a decision boundary obtained by projection pursuit. Density estimated from true classes in one-dimensional subspace minimizing kurtosis coefficient \ref{fig:PP}. Clusters detected by GMM in projected space \ref{fig:PPGMM}. {\ccl} clustering with decision boundary found by one-dimensional GMM \ref{fig:PPc3l}.}\vspace{-0.2in}
	\label{fig:chem} 
\end{figure}

We considered Statlog data set retrieved from UCI repository concerning credit card applications \cite{asuncion2007uci}. Each example is represented by 14 attributes and belongs to one of two classes. First class contains 307 instances, while the second one has 383 objects. 

We used R package REPPlab\footnote{\url{https://github.com/cran/REPPlab/}}, which implements several indices for projecting the data on the associated one-dimensional directions. Since we aim at finding clusters, we chose such a direction, which minimizes the kurtosis coefficient of projection. Figure \ref{fig:PP} presents density estimated from two underlying classes in an optimal one-dimensional subspace. Let us observe that these classes cannot be separated in the reduced space. Given one-dimensional view we applied classical GMM to detect two clusters (see Figure \ref{fig:PPGMM}). The agreement between true classification and GMM clustering was measured by NMI, which gave the score of 0.24. 

GMM clustering in one dimensional space generated linear decision boundary $H$. In order to improve this clustering, we passed $H$ to {\ccl}, which was run with two clusters for four different leakage levels $\alpha \in \{0.01, 0.05, 0.25, 0.5\}$. 

The results presented in Figure \ref{fig:PPc3l} show that applying {\ccl} with low leakage levels led to the improvement of linear decision boundary produced by GMM in projected space. High values of NMI for $\alpha \in \{0.01, 0.05\}$ and its low values for $\alpha \in \{0.25, 0.5\}$ prove that the information retrieved by applying PP was meaningful. In other words, strictly unsupervised model-based clustering could not find true structure of classes, while combining one-dimensional projection with constrained clustering gave significantly better results.

\subsection{Detection of chemical classes} \label{exp:last}

Finally, we considered a real data set of 2497 chemical compounds which was manually clustered into four categories by the experts in the field %\footnote{This is a hierarchical clustering which contains four levels at the highest hierarchy  level.}
 \cite{warszycki2013linear}. Each compound was characterized by its structural features using Klekota-Roth fingerprint (4860 attributes) \cite{klekota2008chemical, smieja2016average}. To reduce a dimensionality of the space, PCA was applied to attribute vectors and only five principle components were used. 

Additionally, every compound was assigned to active or inactive class based on its binding constant $K_i \geq 0$ measured for 5-HT$_{1A}$ receptor, one of the proteins responsible for the regulation of central nervous system: compounds with $K_i < 50$ were considered active while those with $K_i \geq 50$ were treated as inactives \cite{olivier1999}. Summarizing, this data set is contained in $\R \times \R^5$ space and a decision boundary is given by $H=\{50\} \times \R^5$. We investigate whether the information of compounds activity allows to obtain a partition which is more similar to the expert reference grouping with four chemical categories. Observe that this case study represents more realistic scenario than previously prepared experiments.

\begin{figure}[t]
\centering
\subfigure[]{\label{fig:h}\includegraphics[width=2.3in]{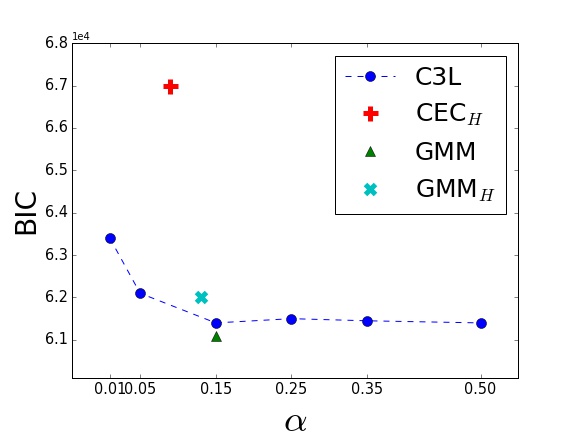}}
\subfigure[]{\label{fig:ll}\includegraphics[width=2.3in]{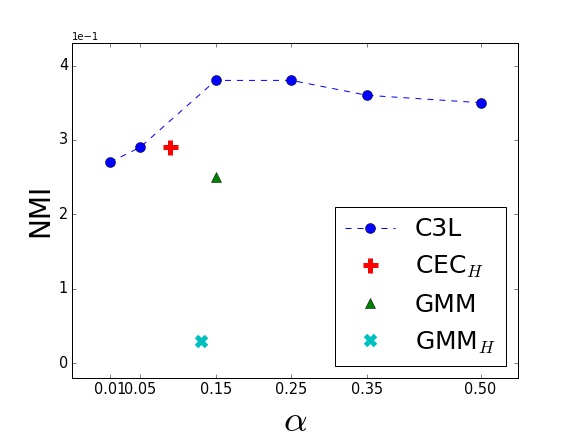}}
\caption{Results for chemical data set.}\vspace{-0.2in}
	\label{fig:chem} 
\end{figure}

Figures \ref{fig:h} and \ref{fig:ll} show that {\ccl} gave the second best model in terms of BIC (worse than GMM). Moreover, the partition obtained by {\ccl} was the most similar to the reference grouping. Observe that high NMI values coincide with the stabilization of BIC (compare Figure \ref{fig:h} with Figure \ref{fig:ll}). The highest similarity was achieved for $\alpha = 0.15$. {\ccl} with lower leakage levels as well as CEC$_H$ and GMM$_H$ gave worse results, because the binding constant does not reflect exactly the chemical classes (as it was prepared in previous experiment). Since the activity barrier represents a kind of noisy decision boundary it should be taken into account with lower confidence level (higher leakage). This case study demonstrated that {\ccl} can be successfully applied in natural machine learning problems.

%This suggests a natural strategy for the selection of the leakage level: one should take the lowest value $\alpha$ which does not improve the clustering model much (however this hypothesis needs further verification. 

\subsection{Summary of the experiments}

Experimental study showed that {\ccl} can be successfully applied in a wide range of problems. If we have a coarse categorization of data into two basic classes, {\ccl} can be used to construct clusters, which agree with both expert categorization and true data distribution. In other words, {\ccl} is capable of detecting interesting subgroups, which belong to one of two classes with a fixed probability of error. The quality of such model was verified by applying internal clustering measures such as BIC (Section \ref{sec:ex1}) as well as by comparing the results with reference grouping created by a domain expert on various examples of data (Section \ref{sec:ex2}).

In the case of partial labeling, where only a small sample of data is categorized, we can first apply any binary classifier to construct approximated decision boundary on the entire data space (Section \ref{sec:ex2}). If constructed classification is meaningful (accurate), then subgroups can be found by applying {\ccl} with low leakage levels, while in more uncertain situations the leakage levels should be higher. In particular, we showed that compounds activity delivers small amount of information about structural division of chemical space (Section \ref{exp:last}).

Finally, we demonstrated that {\ccl} can also be used in strictly unsupervised cases. {\ccl} can improve the results of simpler clustering techniques by introducing nonlinearities in clusters descriptions. In particular, we showed that combining {\ccl} with projection pursuit allows to construct better clustering structure than using projection pursuit or model-based clustering individually (Section \ref{sec:ex3}).

\section{Conclusion and future work}

The paper presented a clustering model, {\ccl}, which integrates information coming from the initial classification with the true structure of data. The idea is based on retrieving Gaussian-like clusters which are contained in one of initial classes within a predefined confidence level. Experimental results prove that our algorithm provides high quality model, which can be used for discovering natural subgroups in partially classified data spaces. In the optimization procedure we restricted the problem to special type of Gaussian densities, which is the main limitation of our model. It is worth to eliminate this assumption in future and extend {\ccl} to arbitrary Gaussian distributions using either analytical calculations or  by applying some numerical procedures.

We also plan to examine its practical usefulness in real-life problems. In particular, we will focus on applying {\ccl} in text translation systems to find clusters, which combine domain knowledge with a distribution of data (see introduction for details). Since Gaussian mixture model does not fit well to high dimensional data, such as text, it might be needed to extend the model to other probability models or to use projections to lower dimensional spaces. In our opinion, combining {\ccl} with projection pursuit is an approach of great practical potential, which needs further studies. While projection pursuit allows to select optimal low dimensional view of data for finding general clusters regions, {\ccl} is able to use this knowledge to discover detailed clusters description.

\appendix 

\section{Algorithm} \label{app:algo}

Given a density model $\G_{1,N-1}$ and a linear constraint $(H, \alpha)$, where $H =\{0\} \times \R^{N-1}$, the one cluster {\ccl} cost function can be evaluated, as a sum of partial costs (see Corollary \ref{cor:reduction}):
$$
E^\alpha_{H}(X \| \G_{1,N-1}) = E^{\alpha}_{\{0\}}(X\usub{1}\|\G_1)+h^\times(X\ussub{2}{N}\|\G_{N-1}).
$$
The optimization of $(N-1)$-dimensional density $g_{N-1} \in \G_{N-1}$ is independent of the constraint. Its optimal parameters are the maximum likelihood estimators (MLE) of a mean and covariance of a cluster, i.e., 
\begin{equation}\label{f1}
g_{N-1} = \nor(\mm{X\ussub{2}{N}}, \SS{X\ussub{2}{N}}).
\end{equation}
The constraint only affects the form of remaining one dimensional density $g_1 \in \G_1$. Making use of the results and the notations of Theorem \ref{thm} it is calculated as:
\begin{equation}\label{f2}
g_1 = \nor(\m{X\usub{1}}{\alpha}, \s{X\usub{1}}{\alpha}).
\end{equation}

The one cluster cost functions of every group are plugged into the expression \eqref{eq:total} and determine the overall {\ccl} cost. Its minimization can be performed in an iterative procedure which is a modified online Hartigan algorithm employed in k-means method \cite{hartigan1979algorithm}. Basically, the procedure consists of two steps: initialization and iteration. In the initialization stage, $k \geq 2$ nonempty groups are formed randomly. Then the elements are reassigned between clusters in order to minimize the criterion function. Presented clustering procedure is non-deterministic and one of the local minima is found \cite{jain1999}. To provide more stable and accurate results, the algorithm has to be run a couple of times and a partition with a minimal cost should be chosen. A pseudocode of {\ccl} is given below:

\begin{flushleft}
\begin{algorithmic}[1] 
\STATE \textbf{INPUT:}
	\STATE $X \subset \R^N$
	\STATE $k$ - number of clusters
	\STATE $(H,\alpha)$ - linear constraint\\
\STATE \textbf{OUTPUT:}\\
	\STATE Final partition $\Y$ of $X$\\
\STATE \textbf{INITIALIZATION:}\\
		\STATE $X = T_H(X)$ \COMMENT{data transformation such that $H$}\\
		\COMMENT{ is mapped into $(\{0\} \times \R^{N-1})$}
		\STATE $\Y \leftarrow$ random partition of $X$ into $k$ groups 
	\FORALL{$Y \in \Y$}
		\STATE $g^Y_{N-1} \leftarrow MLE(Y\ussub{2}{N} \| \G_{N-1})$ \COMMENT{use standard MLE to find optimal density \eqref{f1}}
		\STATE $g^Y_{1} \leftarrow ConstrMLE(Y\usub{1} \| \G_{1} \text{ s.t. } (\{0\} \times \R^{N-1}, \alpha))$ \COMMENT{use Theorem \ref{thm} for density estimation \eqref{f2}}
		\STATE $g^Y \leftarrow g^Y_1 \cdot g^Y_{N-1}$
	\ENDFOR \\
\STATE \textbf{ITERATION:}\\
	\WHILE{NOT Done}
		\STATE $Done \leftarrow True$
		\FORALL{$x \in X$}
			\STATE $Y_{new} \leftarrow arg \max\limits_{Y \in \Y} \Delta E_{\{0\} \times \R^{N-1}}^\alpha(x,Y)$ \COMMENT{find a membership of $x$ maximizing the decrease of cost}\label{lineCost}
			\IF{$Y_{new} \neq x.cluster$}
				\STATE $Done \leftarrow False$
				\STATE $Reassign(x, Y_{new}, Y_{old})$ \COMMENT{reassign $x$ from $Y_{old}$ to $Y_{new}$}
				\STATE $Update((Y_{new}, g^{new}), (Y_{old}, g^{old}), x, \alpha)$ \COMMENT{recalculate clusters parameters after the change} \label{lineUpdate}
			\ENDIF
		\ENDFOR
	\ENDWHILE
\end{algorithmic}
\end{flushleft}

\section{Proof of Lemma \ref{malyLem}} \label{app:lem}

We consider the limiting case of $\alpha \to 0$. Therefore, without loss of generality we may assume that there exists $\alpha_0 > 0$ such $\m{X}{\alpha}$ and $\s{X}{\alpha}$, for $\alpha < \alpha_0$, are given by \eqref{eq:optUnivariate}, i.e.,
$$
\begin{array}{ll}
\m{X}{\alpha} & := \frac{-(\k)^2 \mm{X}+\sgn(\mm{X}) \k\sqrt{((\k)^2+4)\mm{X}^2 + 4\ss{X}^2}}{2},\\
\s{X}{\alpha} & := \frac{|\mm{X}^\alpha|}{\k},
\end{array}
$$
where 
$$
\k = \Phi_{\nor(0,1)}^{-1}(1-\alpha).
$$  

Let us calculate the limiting value of $\m{X}{\alpha}$:
$$
\begin{array}{l}
2 \m{X}{\alpha} = \\
 -(\k)^2 \mm{X}+\sgn(\mm{X})\k\sqrt{((\k)^2+4)\mm{X}^2 + 4\ss{X}^2} =\\
 -(\k)^2 \mm{X}+(\k)^2 |\mm{X}| \sgn(\mm{X}) \sqrt{1+\frac{4}{(\k)^2} + \frac{4\ss{X}^2}{(\k)^2 \mm{X}^2}} =\\
 -(\k)^2 \mm{X}+(\k)^2 \mm{X} \sqrt{1+\frac{4}{(\k)^2}\left( 1 + \frac{\ss{X}^2}{\mm{X}^2}\right)}.\\
\end{array}
$$
Making use of Taylor expansion we have
$$
\sqrt{1+x} = 1+\frac{x}{2} + \varepsilon_x \text{ , for } \varepsilon_x = -\frac{1}{8}(1+\xi_x)^{-\frac{3}{2}} \xi_x^2, 
$$ 
where  $\xi_x \in (0,x)$. Consequently, in our situation there exists $\xi_\alpha \in (0, \frac{4}{(\k)^2}( 1 + \frac{\ss{X}^2}{\mm{X}^2}))$ such that for $\varepsilon_\alpha = -\frac{1}{8}(1+\xi_\alpha)^{-\frac{3}{2}} \xi_\alpha^2$ we get
$$
\begin{array}{l}
2 \m{X}{\alpha} =\\
 -(\k)^2 \mm{X}+(\k)^2 \mm{X}\sqrt{1+\frac{4}{(\k)^2}( 1 + \frac{\ss{X}^2}{\mm{X}^2})} =\\
 -(\k)^2 \mm{X}+(\k)^2 \mm{X} \left(1 + \frac{1}{2}  \frac{4}{(\k)^2} \left( 1 + \frac{\ss{X}^2}{\mm{X}^2}\right) + \varepsilon_\alpha \right)=\\
\underbrace{(\k)^2 \mm{X} \frac{2}{(\k)^2}\left( 1 + \frac{\ss{X}^2}{\mm{X}^2}\right)}_{\text{(I)}} + \underbrace{(\k)^2 \mm{X} \varepsilon_\alpha}_{\text{(II)}},
\end{array}
$$
Clearly, (I) converges to $2 (\mm{X} + \frac{\ss{X}^2}{\mm{X}})$, as $\alpha \to 0$. We consider the term (II). Observe that for sufficiently small $\alpha > 0$, we have:
$$
\begin{array}{ll}
|(\k)^2 \mm{X} \varepsilon_\alpha|& =  \frac{1}{8} (\k)^2 |\mm{X}| (1+\xi_\alpha)^{-\frac{3}{2}} \xi_\alpha^2 \\
& \leq  \frac{1}{8} (\k)^2 |\mm{X}| \cdot 1 \cdot \frac{16}{(\k)^4}( 1 + \frac{\ss{X}^2}{\mm{X}^2})^2\\
 & \leq 2 |\mm{X}| \frac{1}{(\k)^2}( 1 + \frac{\ss{X}^2}{\mm{X}^2})^2 \xrightarrow{ \alpha \to 0 } 0.
\end{array}
$$
Concluding $\m{X}{\alpha} \to \mm{X} + \frac{\ss{X}^2}{\mm{X}}$, as $\alpha \to 0$.

From the above calculation we directly get,
$$
\s{X}{\alpha} = \frac{\m{X}{\alpha}}{\k} \xrightarrow{\alpha \to 0} 0,
$$
which completes the proof.

\section{Proof of Lemma \ref{lemLast}} \label{app:lem1}

Let $\alpha > 0$. The violation of the linear constraint $(\{0\}, \alpha)$ by a density $\nor(m,\sigma)$ is verified by the condition:
$$
|m| \geq \k \sigma \text{, where } \k = \Phi_{\nor(0,1)}^{-1}(1 - \alpha).
$$

Theorem \ref{thm} states that the parameters of optimal one dimensional Gaussian density are calculated as MLE within the cluster if only it does not violate the linear constraint. However, in the limiting case there exists $\alpha_0$ such that for all $\alpha < \alpha_0$ the constraint is violated and consequently the formulas \eqref{eq:optUnivariate} for $\m{\pm}{\alpha}, \s{\pm}{\alpha}$ give optimal solutions.

Let $x> 0$. For normal distributions $g_\pm^\alpha = \nor(\m{\pm}{\alpha}, \s{\pm}{\alpha})$ we have
$$
\begin{array}{l}
\ln g^\alpha_\pm(x)  = \\ -\frac{1}{2}\ln (2\pi) - \ln(\s{\pm}{\alpha}) - \frac{(x - \m{\pm}{\alpha})^2}{2 (\s{\pm}{\alpha})^2}=\\
 - \frac{x^2}{2 (\s{\pm}{\alpha})^2} + \frac{2 x \m{\pm}{\alpha}}{2 (\s{\pm}{\alpha})^2} - \frac{(\m{\pm}{\alpha})^2}{2 (\s{\pm}{\alpha})^2} - \ln(\s{\pm}{\alpha}) -\frac{1}{2}\ln (2\pi). \\
\end{array}
$$

Making use of equality $\sigma_\pm^2 = \frac{m_\pm^2}{(\k)^2}$ we get
$$
\begin{array}{l}
\ln g^\alpha_\pm(x)  = \\
\frac{1}{2} \left( - \frac{(\k)^2}{(m_\pm^\alpha)^2} x^2  + \frac{2 (\k)^2}{m_\pm^\alpha} x - (\k)^2 - \ln \frac{(m_\pm^\alpha)^2}{(\k)^2} - \ln (2\pi)\right),
\end{array}
$$
and consequently
$$
\begin{array}{l}
\ln g^\alpha_+(x) -\ln g^\alpha_-(x)  = \\
%\\- \frac{(\k)^2}{(\m{+}{\alpha})^2} x^2  + \frac{2 (\k)^2}{\m{+}{\alpha}} x - \ln \frac{(\m{+}{\alpha})^2}{(\k)^2} + \frac{(\k)^2}{(\m{-}{\alpha})^2} x^2  - \frac{2 (\k)^2}{\m{-}{\alpha}} x + \ln \frac{(\m{-}{\alpha})^2}{(\k)^2}=\\
\frac{(\k)^2 x}{(\m{+}{\alpha})^2 (\m{-}{\alpha})^2} \left( -x (\m{-}{\alpha})^2 +  x (\m{+}{\alpha})^2 + \right. \\ 
\left. 2 \m{+}{\alpha} (\m{-}{\alpha})^2 - 2 (\m{+}{\alpha})^2 \m{-}{\alpha} \right) - \ln \frac{(\m{+}{\alpha})^2}{(\m{-}{\alpha})^2}.
\end{array}
$$

Since $\m{\pm}{\alpha} \to \mm{\pm} + \frac{\ss{\pm}^2}{\mm{\pm}}$, as $\alpha \to \infty$ (see Theorem \ref{maleTw}) then
$$
0 < \frac{x}{(\m{+}{\alpha})^2 (\m{-}{\alpha})^2} < \infty \text{ and }
0 < \left| \ln \frac{(\m{+}{\alpha})^2}{(\m{-}{\alpha})^2}\right| < \infty
\text{, as } \alpha \to 0.
$$
Therefore,
$$
\ln g^\alpha_+(x) -\ln g^\alpha_-(x)  \to +\infty,
$$ 
iff
\begin{equation}\label{llast}
 -x (\m{-}{\alpha})^2 + x (\m{+}{\alpha})^2 + 2 \m{+}{\alpha} (\m{-}{\alpha})^2 - 2 (\m{+}{\alpha})^2 \m{-}{\alpha} > 0 \text{, as } \alpha \to 0.
\end{equation}
The inequality \eqref{llast} holds iff
$$
\begin{array}{ll}
\text{either: } & x > 2 \m{+}{\alpha}\text{ and } \m{-}{\alpha} > \frac{\m{+}{\alpha} x}{2 \m{+}{\alpha} - x} \\
\text{or: } & x \leq 2 \m{+}{\alpha}.
\end{array}
$$
In the limiting case the last inequality expands to:
$$
x \leq 2 \m{+}{\alpha} \xrightarrow{\alpha \to 0} 2 \left( \mm{+} + \frac{\ss{+}^2}{\mm{+}} \right),
$$
which completes the proof.

\section*{Acknowledgement}

This research was partially supported by the National Science Centre (Poland) grant no. 2016/21/D/ST6/00980 and grant no. 2015/19/B/ST6/01819.

\bibliographystyle{plain}
\bibliography{bib}   % name your BibTeX data base

\end{document}